\documentclass[letterpaper, 10 pt, conference]{ieeeconf}

\IEEEoverridecommandlockouts                           
\usepackage{graphics}
\usepackage{epsfig} 

\usepackage{amsmath} 
\usepackage{amssymb} 

\usepackage[ruled, linesnumbered]{algorithm2e}
\usepackage{subfigure}
\usepackage{amsthm}
\usepackage{stfloats}
\usepackage{adjustbox}

\usepackage{cite}
\usepackage{multirow}

\usepackage{hyperref}
\hypersetup{
    colorlinks=true,
    linkcolor=blue,      
    urlcolor=blue
}

\newtheorem{thm}{Theorem}

\newcommand{\argmin}{\mathop{\mathrm{argmin}}}

\title{\LARGE \bf
Certifiable Safe Model-Based Reinforcement Learning with Control-Affine Dynamics Approximation}

\author{
Hao Zhou$^1$, Yanze Zhang$^1$, Cameron Reid$^1$, and Wenhao Luo$^2$
\thanks{$^*$This work was supported in part by the U.S. National Science Foundation under Grant 2530297.}
\thanks{
$^1$Authors are with the Department of Computer Science, University of Illinois Chicago, Chicago, IL 60607, USA.
Email: {\tt\small \{hzhou134, yzhan361, creid22\}@uic.edu}
}
\thanks{
$^2$The author is with the Department of Computer Science and Engineering, Texas A\&M University, College Station, TX 77843, USA. \texttt{wenhaol@tamu.edu}
}
}

\begin{document}

\maketitle
\thispagestyle{empty}
\pagestyle{empty}


\begin{abstract}
Safe model-based reinforcement learning (RL) often bridges control-theoretic analysis and RL for robots to safely explore (partially) unknown system dynamics while deriving control actions for task efficiency. The control performance and safety assurance typically rely on prior knowledge of partially modeled nominal system dynamics and the data-driven models that compensate for residual model uncertainties. However, existing methods often overlook the structure of residual model uncertainties (e.g., components affine in control), which could lead to overly conservative robot behaviors or invalid safety guarantees under the safe learning-based controllers. This paper proposes a safe reinforcement learning framework that learns control-affine dynamics with a certifiable data-driven safe policy using control barrier functions (CBF). Specifically, we first use Control-Affine Random Fourier Features (ARFF) to model robot dynamics in a control-affine form, which offers computational efficiency that scales with dataset size and reduces potential model bias for model-based reinforcement learning. Then, a model-free, efficient uncertainty quantification method using adaptive conformal prediction (ACP) is applied to quantify the uncertainty in the safety constraint arising from the learned control-affine dynamics. This allows for data-driven safety assurance amenable to principled and efficient controller synthesis with CBF. Simulation results on the cartpole and the 3D quadrotor platforms demonstrate the effectiveness of the proposed framework.
\end{abstract}

\section{Introduction}

Reinforcement Learning (RL) has been successfully applied to various robotics tasks \cite{chowdhury2019online,chua2018deep,curi2020efficient,deisenroth2011pilco,kamthe2018data}, in which a robot learns to perform a task optimally through interaction with the environment. In many real-world applications, however, the environmental uncertainty and limited prior information about the robot's dynamics could easily render unsafe behaviors during learning. This could often lead to catastrophic consequences and thus calls for safe reinforcement learning (safe RL) \cite{brunke2022safe} that trains a policy to maximize task efficiency while ensuring safety during learning and deployment. 

Safe RL is usually achieved under the model-based reinforcement learning (MBRL) framework \cite{luo2022sample, berkenkamp2017safe} or model-free reinforcement learning (MFRL) framework \cite{cheng2019end} with the addition of the safety constraint. MBRL first learns the dynamics and then uses them to obtain the policy, making it more sample-efficient than MFRL. Apart from the sample efficiency of the MBRL, the dynamics obtained from the MBRL can be embedded into the optimal control framework (e.g., model predictive control\cite{kamthe2018data, chua2018deep}) or the safety filter method (e.g., control barrier function \cite{wang2018safe, zhou2024safety,cheng2019end}) to guarantee safety. Due to the sample efficiency and the convenience of constructing optimization problem using the learned dynamics for the safety guarantee, MBRL will be adopted in this paper as the fundamental framework.

Gaussian Processes (GPs) \cite{schulz2018tutorial}, as a data-driven model for learning dynamics, have been widely used in MBRL \cite{curi2020efficient, wang2018safe, berkenkamp2017safe} since they can provide both mean and variance of the prediction for uncertainty-aware decisions, e.g., using variance to quantify uncertainty of the safety constraint for safe exploration \cite{curi2020efficient}. However, GPs incur a high computational burden during training, which limits their applicability to online control. Approximating the GPs via the Random Fourier Feature (RFF) or Sparse GPs is computationally efficient, but this potentially provides the biased prediction \cite{potapczynski2021bias} and hence affects the performance of the MBRL. 

Apart from modeling the dynamics, safe MBRL aims to obtain a safe policy, e.g., through control barrier functions (CBF) \cite{zhou2025computationally, luo2022sample, emam2022safe, cheng2019end} that enforce safety-critical control constraints while learning the system dynamics.
However, deriving effective CBF constraints with learned dynamics often rely on the assumption that the control-affine component in the dynamics is fully known and the unmodelled residual part depends only on the state \cite{wang2018safe, cheng2019end, luo2022sample}. 
This restrictive assumption could significantly limit the applicability of the resulting policy, introduce a bias in the learned dynamics, and may yield invalid safety performance.

On the other hand, approximating GPs using RFF usually underestimates uncertainty (variance) \cite{li2024trigonometric}, leading to an overconfident uncertainty-aware safety constraint and, consequently, failing to guarantee safety. To effectively calibrate uncertainty, statistical methods such as conformal prediction (CP) \cite{papadopoulos2002inductive} can be used to quantify the predictive uncertainty, thus enabling robust safe control \cite{lindemann2023safe}. 
Although CP can quantify uncertainty without depending on the predictive models, it assumes the testing data is exchangeable with those in a pre-collected calibration dataset, which can be violated easily in the online setting where robots continuously collect data and refine the learned model. 

To address these challenges, 
in this paper we propose a novel safe RL approach for a robot to safely learn and approximate its control-affine dynamics with adaptive recalibrated uncertainty for certifiable high-probability safety assurance during learning and deployment.
We will first leverage the control-affine random fourier features (ARFF) \cite{kazemian2024random} originally developed for learning control Lyapunov function to characterize the data-driven system dynamics and the control barrier functions without biasing the control-affine structure.
Subsequently, we use adaptive conformal prediction (ACP) \cite{dixit2023adaptive} to quantify the uncertainty when enforcing the learned safety constraints as robots collect data during learning in an online setting. 
Moreover, we integrate the learned control-affine dynamics and the uncertainty-aware CBF into an MBRL framework with optimism-based exploration, enabling safe exploration and achieving near-optimal control performance. Our \textbf{contributions} are threefold: (1) We incorporate a control-affine approximation approach within an MBRL framework to explicitly preserve the control-affine structure when learning the unknown system dynamics and constructing data-driven CBF-based safe control constraints.
(2) We develop robustified safe control constraints by accounting for predictive uncertainty via ACP and provide theoretical analysis to justify the certifiable safety performance. 
(3) We present simulation results on the cartpole and the quadrotor platforms to demonstrate the effectiveness of our proposed method against the baselines in terms of control performance and safety assurance.

\begin{figure}[t]
	\centering
		\centering
		\includegraphics[scale=1, width=1.0\linewidth]{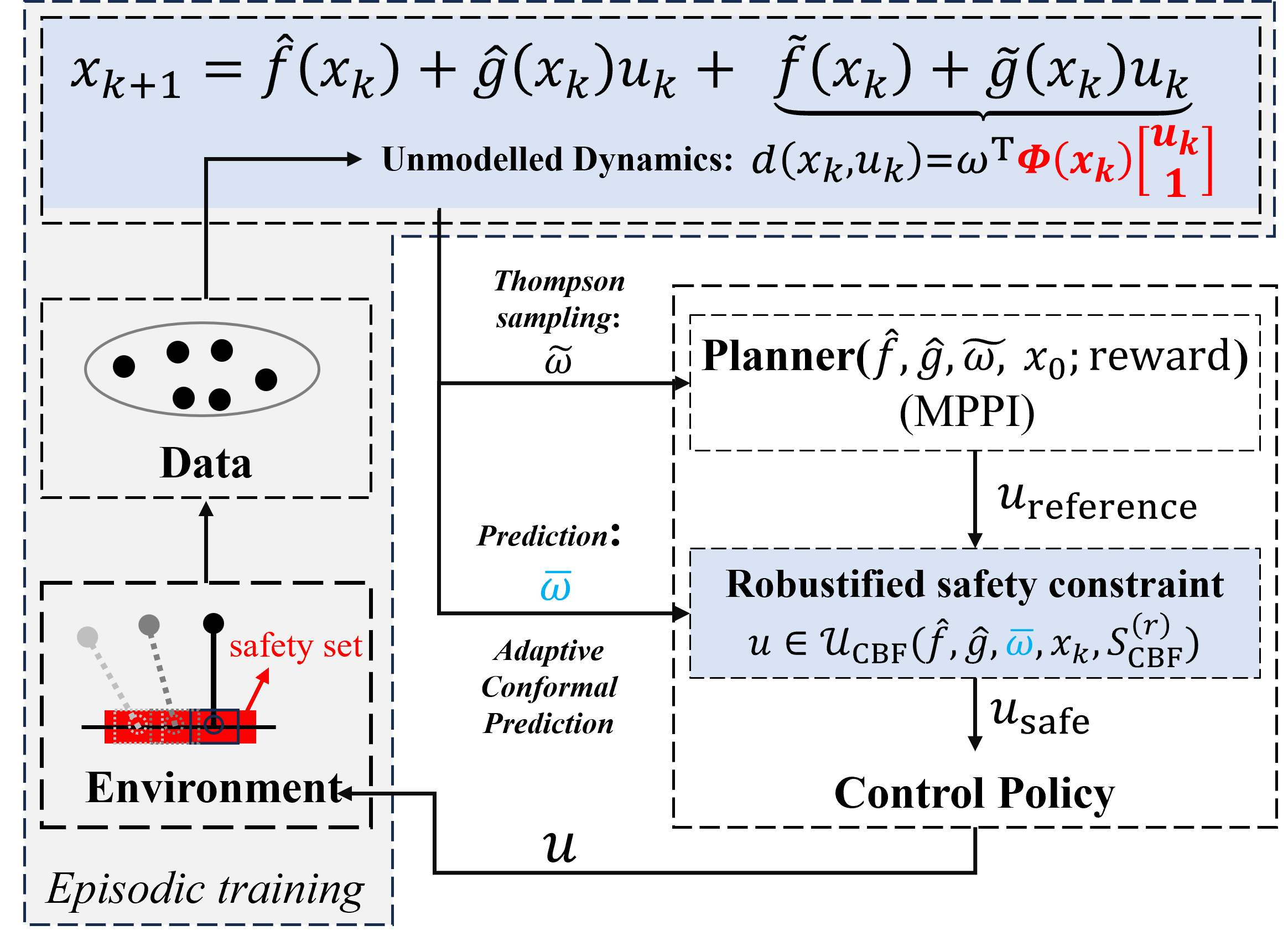}
	\caption{The episodic framework for safe model-based reinforcement learning using control-affine dynamics approximation. 
    For each episode, model predictive path integral (MPPI) as a model-based planner \cite{williams2018information,williams2017information} uses sampled dynamics via Thompson sampling to generate the reference control input. Then, uncertainty-aware safety constraint based on adaptive conformal prediction is leveraged to guarantee safety.}
	\label{fig: framework}
    \vspace{-0.6cm}
\end{figure}

\section{Preliminary}
\subsection{Safe Learning Objective}
In this paper, we investigate safe model-based reinforcement learning due to its sample efficiency compared to model-free reinforcement learning. The unknown dynamics required to be learned under model-based reinforcement learning is defined as,
\begin{equation}\label{eq: dyn-d}
    x_{k+1} = \hat{f}(x_k) + \hat{g}(x_k)u_k + d(x_k, u_k)
\end{equation}
where $x_k\in \mathbb{R}^n$ and $x_{k+1}\in \mathbb{R}^n$ are the robot states at current and next time steps. $u_k\in\mathbb{R}^m$ is the robot control input and $d: \mathbb{R}^n\times\mathbb{R}^m\mapsto\mathbb{R}^n$ is the unknown dynamics that must be learned. $\hat{f}(x_k) + \hat{g}(x_k)u_k$ is the control-affine nominal robot dynamics with $\hat{f}:\mathbb{R}^n\mapsto\mathbb{R}^n$ and $\hat{g}:\mathbb{R}^n\mapsto\mathbb{R}^{n\times m}$, which can be derived from the underlying physical laws.

Control-affine dynamics are widely used in robotics, such as in robotic manipulators, cartpoles, and quadrotors. Previous work\cite{wang2018safe, curi2020efficient, cheng2019end} leveraged Gaussian Process or neural networks to learn unknown dynamics $d(x_k, u_k)$ with state-only input $x_k$, which ignores the control-affine form of $d(x_k, u_k):=\tilde{f}(x_k)+\tilde{g}(x_k)u_k$ where $\tilde{f}:\mathbb{R}^n\mapsto\mathbb{R}^n$, $\tilde{g}:\mathbb{R}^n\mapsto\mathbb{R}^{n\times m}$. Without considering the control-affine structure in $d(x_k, u_k)$, the safe policy obtained by CBF under the reinforcement learning framework may not be solved efficiently and optimally. We discuss how to learn the control-affine component $d(x_k, u_k)$ in section \ref{sec: control affine d} and how to embed it into the safe reinforcement learning framework in \ref{sec: algorithm framework and analysis}.

To learn unknown control-affine dynamics $d(x_k, u_k)$ while guaranteeing the primary control task, a widely used approach \cite{luo2022sample, kakade2020information} is to learn $d(x_k, u_k)$ by optimizing the objective function below,

\begin{equation}\label{eq: learning objective}
    \min_{\pi \in \Pi} \min_d J^{\pi}(x_0; R, d) = \min_{\pi \in \Pi} \min_{d} \mathbb{E} [\sum_{k=0}^{K-1} R(x_k, u_k) | \pi, x_0, d]
\end{equation}

where $R\in \mathbb{R}^{+}$ and $\pi$ represent the stage cost function and the control input, respectively. The total time horizon within each learning episode is $K$. 
To guarantee safety, the safety constraint is embedded into the Eq.~\eqref{eq: learning objective} as follows
\begin{equation}\label{eq: safe learning opt}
	\begin{aligned}
		\min_{\pi \in \Pi} \min_{d} J^{\pi}(x_0; R, d) 
		\quad s.t. \: \mathbb{P}(c(x_k, u_k) \geqslant 0) \geqslant 1-\alpha
	\end{aligned}
\end{equation} 

where $c:\mathbb{R}^n \times \mathbb{R}^m \mapsto \mathbb{R}$ is the safety constraint function with 
$c(x_k, u_k)\geqslant 0$ defining safety condition. We model $c(x_k,u_k)$ as a chance constraint because of the epistemic uncertainty from the learning model of $d(x_k, u_k)$. $\alpha \in (0, 1)$ is the user-defined failure probability bound of the safety constraint.

\subsection{Control Barrier Functions}
Discrete-time Control Barrier Function (DT-CBF) has already been successfully applied to reinforcement learning to guarantee safety \cite{cheng2019end, zhou2024safety}. In this paper, we use the DT-CBF as a safety constraint to achieve safe reinforcement learning, as it renders the predefined safe set forward invariant. The DT-CBF under the deterministic system is summarized below.

\textbf{Lemma. 1}
\label{lemma: DT-CBF}
    (Discrete-Time Control Barrier Function (DT-CBF) \cite{agrawal2017discrete})  
    Given the deterministic dynamics $\bar{x}_{k+1}=f(\bar{x}_k)+g(\bar{x}_k)u_k$ where $f:\mathbb{R}^n\mapsto\mathbb{R}^n$ and $g:\mathbb{R}^n\mapsto\mathbb{R}^{n\times m}$, the DT-CBF ensures that the system state remains within a safe set $\mathcal{C} \subset \mathbb{R}^n$, which is defined as,  
    \begin{equation}
        \mathcal{C} \triangleq \{\bar{x}_k \in \mathbb{R}^n \mid h(\bar{x}_k) \geqslant 0\}
    \end{equation}  
    Here $h: \mathbb{R}^n \to \mathbb{R}$ is a continuous safety constraint function. If the initial state satisfies $x_0 \in \mathcal{C}$, then the system state remains in $\mathcal{C}$ for all future time steps, provided that the control input $u_k$ belongs to the set  
    \begin{equation}
        \mathcal{H}(\bar{x}_k) = \{u_k \mid h(f(\bar{x}_k)+g(\bar{x}_k)u_k)) \geqslant (1-\gamma) h(\bar{x}_k) \}
    \end{equation}  
    where $\gamma \in (0, 1]$. $h(f(\bar{x}_k)+g(\bar{x}_k)u_k)) \geqslant (1-\gamma) h(\bar{x}_k)$ is named as safety barrier certificates that ensures the forward invariance of the safe set $\mathcal{C}$. DT-CBF still requires the dynamics to be control-affine in order to solve the optimization problem efficiently \cite{agrawal2017discrete}.

\subsection{Adaptive Conformal Prediction}\label{sec: acp}
Conformal Prediction (CP) \cite{papadopoulos2002inductive} is a statistical method to obtain a confidence interval $\mathcal{I}$ with no dependence on the prediction model $G:\mathbb{R}^n \mapsto\mathbb{R}$ under the exchangeability assumption of the dataset. Given a testing input $x^*$, CP can guarantee the marginal probability as below,
\begin{equation}\label{eq: cp}
    \mathbb{P}(G(x^*) \in \mathcal{I}) \geqslant 1-\alpha^c
\end{equation}
where $\alpha^c \in (0,1)$ is the pre-defined failure probability bound. The confidence interval $\mathcal{I}$ is constructed by a quantile number $S^{(r)}\in \mathbb{R}^{+}$ and hence $\mathcal{I}=[G(x^*)-S^{(r)}, G(x^*)+S^{(r)}]$. The quantile number $S^{(r)}$ is obtained by the nonconformity score (similar to the residual error in machine learning) computed in the dataset.
$r:= \lceil (\overline{K}+1)(1-\alpha^c) \rceil$ where $\overline{K}$ is the size of the dataset and $\lceil \cdot \rceil$ is the ceiling function.

\textbf{Adaptive Conformal Prediction (ACP).} CP relies on the assumption of data exchangeability \cite{papadopoulos2002inductive}, which may be easily violated in dynamical time series prediction tasks \cite{gibbs2021adaptive}. To achieve reliable dynamical time series prediction over long time horizons, ACP \cite{gibbs2024conformal, dixit2023adaptive} was proposed to dynamically update the quantile number online. The estimating process is shown below:
	\begin{equation}
		\label{eq: ACP update}
		\alpha_{k+1} = \alpha_{k} + \delta (\alpha^c-e_k) \; \mathrm{with}\; e_k=
		\begin{aligned}
			\begin{cases}
				 1, \mathrm{if}\; S_k^{(r)} < S_{k}, \\
				 0, \mathrm{if}\; S_k^{(r)} \geqslant S_{k}.
			\end{cases}
		\end{aligned}
	\end{equation}
where $S_k \in \mathbb{R}^{+}$ denotes the nonconformity score at time step $k$, and $S_k^{(r)}$ is the corresponding quantile.  If $S_k \leq S_k^{(r)}$, then $\alpha_{k+1}$ is increased, resulting in a larger prediction region $\mathcal{I}$ induced by $S_k^{(r)}$. $\delta$ is the user-specified learning rate.

This paper will use ACP to quantify the uncertainty of the chance constraint (Eq.~\eqref{eq: safe learning opt}) and then derive an uncertainty-aware safety constraint to solve Eq.~\eqref{eq: safe learning opt}. The uncertainty quantification method proposed in this paper is simple, easy to implement, and computationally inexpensive, as discussed in sections \ref{sec: uncertainty quantification} and \ref{sec: algorithm framework and analysis}.

\section{Method}
\subsection{Learning Unknown Control-Affine Dynamics}\label{sec: control affine d}
To learn the unknown dynamics $d(x_k, u_k)\in \mathbb{R}^n$ under the safe learning optimization problem (Eq.\eqref{eq: safe learning opt}), Gaussian Processes have been widely used to model unknown dynamics in this setting \cite{curi2020efficient, wang2018safe} because they naturally provide prediction variance, which can be used both to efficiently explore the environment and to formulate uncertainty-aware safety constraints. However, the Gaussian Processes for $d(x,u)$ do not scale well to large datasets, making training the model computationally intensive. Apart from the computation, \cite{curi2020efficient, wang2018safe} ignores the control-affine form of the unknown dynamics and therefore may sacrifice performance due to intrinsic model error (not control-affine $d(x_k, u_k)$).

To efficiently model $d(x,u)$ by exploiting control-affine structure, we will use control-Affine Random Fourier Features (ARFF) to model the unknown dynamics $d(x,u)$. Motivated by \cite{kazemian2024random}, the ARFF for $d(x,u)$ is defined as,
\begin{equation}\label{eq: d}
    d(x_k, u_k) = \omega^\top \phi(x_k, u_k)
\end{equation}
where $\omega \in \mathbb{R}^{P \times n}$ represents the weight matrix associated with the linear model. $P$ is the number of features and $n$ is the dimension of the robot state. $\phi(x_k, u_k): \mathbb{R}^n \times \mathbb{R}^m \mapsto \mathbb{R}^P$ is the ARFF, which has the following form,
\begin{equation}\label{:eq: affine RFF}
    \phi(x_k, u_k)=[\varphi_1(x_k),\cdots,\varphi_{m+1}(x_k)]\begin{bmatrix}
u_k \\
1
\end{bmatrix}
\end{equation}

Here $\varphi_i(x_k):\mathbb{R}^n \mapsto \mathbb{R}^P$ is the robot state-based random Fourier function (RFF) defined as,

\begin{equation} \label{eq: RFF}
\begin{aligned}
    \varphi_i(x_k) =
    \sqrt{\frac{2}{P}}[\sin{\theta_{i, 1}^\top x_k} \;, \cos{\theta_{i, 1}^\top x_k} \;, \cdots ,\\
    \sin{\theta_{i, \frac{P}{2}}^\top x_k} \; , \cos{\theta_{i, \frac{P}{2}}^\top x_k}]^\top
\end{aligned}
\end{equation}

where $i=1, \cdots, m+1$ and the weight $\{\theta_{i, j}\}_{j=1}^{P/2}$ is drawn i.i.d from the Gaussian distribution.

The training data, collected through safe interaction with the environment within a model-based reinforcement learning framework, are used to train Eq.~\eqref{eq: d}. Unlike the computationally intensive approach based on Gaussian Processes, the parameter $\omega$ in Eq.~\eqref{eq: d} can be obtained analytically via ridge regression, whose computational cost does not scale with the dataset size. In addition to the computational benefits, the use of ARFF also allows for efficiently solving the safe reinforcement learning optimization problem (Eq.\eqref{eq: safe learning opt}). Further details are provided in section \ref{sec: calibrate model}.

ARFF is applied in \cite{kazemian2024random} for learning control Lyapunov functions under hand-collected data. In contrast to \cite{kazemian2024random}, this paper focuses on learning unknown dynamics within a safe reinforcement learning framework, and thus, the data can be automatically obtained.

\subsection{Uncertainty Quantification for Safety Constraint} \label{sec: uncertainty quantification}
To address safety constraints in safe reinforcement learning (Eq.\eqref{eq: safe learning opt}), the CBF will be applied to guarantee safety. However, the epistemic uncertainty of the unknown dynamics $d$ invalidates the safety guarantee under the CBF because the CBF is often defined under known deterministic dynamics. On the other hand, unlike the uncertainty calibration of the Gaussian process applied to quantify the uncertainty of $d$, this paper leverages ARFF under the regression model(Eq.\eqref{eq: d}) to approximate the Gaussian Process for computational efficiency, and hence the uncertainty estimation from the regression model is usually underestimated\cite{potapczynski2021bias}. The underestimated uncertainty used to quantify the safe constraint may not guarantee safety. 

To address the underestimated uncertainty in the regression model (Eq. \eqref{eq: d}), we use adaptive conformal prediction to quantify the uncertainty, which is online, computationally efficient, and yields a reliable confidence interval that does not depend on the model choice.

Let $G(x_k, u_k)\triangleq h(\hat{f}(x_k) + \hat{g}(x_k)u_k+d(x_k, u_k)) - (1-\gamma)h(x_k)$, where $G(x_k, u_k)$ denotes the predicted barrier certificates value. $d(x_k, u_k)$ is a control-affine surrogate model learned through safe exploration, which captures the unknown dynamics.

Let $G^{*}(x_k, u_k) \triangleq h(f(x_k) +g(x_k)u_k)) - (1-\gamma)h(x_k)$ represent the true value of prediction model $G(x_k, u_k)$. From Eq.(\ref{eq: cp}) in section \ref{sec: acp}, we know that,
\begin{align}\label{Eq: prob B ACP expand}
    \mathbb{P}( -S^{(r)} + G(x_k, u_k) \leqslant G^{*}(x_k, u_k) \leqslant S^{(r)} &+ G(x_k, u_k) ) \notag \\
    &\geqslant 1-\alpha
\end{align}
where $S^{(r)}$ denotes the quantile of the nonconformity score set $\mathcal{S}$, which consists of the nonconformality scores $S_k=|G^{*}(x_k, u_k) - G(x_k,u_k)|$ computed at each time step.  

\begin{thm}\label{thm}
Suppose there exists a failure probability $\alpha \in (0, 1)$. Consider the system dynamics in Eq.~\eqref{eq: dyn-d} and a learning horizon of length $K$. Let $S^{(r)}$ denote the corresponding quantile of the nonconformity score, where $r:= \lceil (K+1)(1-\alpha^c) \rceil$. Then, enforcing Eq.~\eqref{eq: safe learning opt} with a probability guarantee of $1-\alpha$ yields the following constraint:
    \begin{equation} \label{eq: prob B to B}
        \mathbb{P}(G^{*}(x_k, u_k)\geqslant 0) \geqslant 1-\alpha \Rightarrow -S^{(r)} + G(x_k, u_k) \geqslant 0
    \end{equation}
\end{thm}

\begin{proof}
    Since the true prediction value $G^*(x_k, u_k)$ is not directly accessible, we use the prediction model $G(x_k, u_k)$ to construct a confidence region which $F^{*}(x_k, u_k)$ belongs to. By the ACP, it holds that $\mathbb{P}( -S^{(r)} + G(x_k, u_k) \leqslant G^{*}(x_k, u_k) \leqslant S^{(r)} + G(x_k, u_k) ) \geqslant 1-\alpha$. Therefore, a sufficient condition for $G^{*}(x_k, u_k)\geqslant 0$ to hold with probability at least $1-\alpha$ is that the lower bound of the interval satisfies $-S^{(r)} + G(x_k, u_k) \geqslant 0$. 
\end{proof}

Theorem \ref{thm} translates the chance constraint into the uncertainty-aware deterministic constraint using ACP, which is simple and efficient because we only need to compute the quantile number $S^{(r)}$. Theorem \ref{thm} is also beneficial for solving the optimization problem (Eq.\eqref{eq: safe learning opt}) because the previous method (Gaussian Process) requires propagating the optimization variable through the kernel function to obtain the covariance matrix, which may render the optimization problem nonlinear. Using Theorem \ref{thm}, we don't need to propagate the optimization variable because $S^{(r)}$ is a constant inside the optimization problem (Eq.\eqref{eq: safe learning opt}).

\section{Algorithms and Analysis}

\subsection{Model Calibration}\label{sec: calibrate model}
In section \ref{sec: control affine d}, the unknown dynamics $d(x_k, u_k)$ is modeled as $d(x_k, u_k)=\omega^\top \phi(x_k, u_k)$, which translates the unknown dynamics into the control-affine form automatically. To avoid overfitting, we will use the ridge regression to obtain the $\overline \omega_t$ under each training episode $t$ as below,

\begin{equation}\label{eq: ridge regression}
    \overline{\omega}_t=\argmin_{\omega} \sum_{i=1}^{t-1} \sum_{k=1}^{K} || \phi^\mathrm{T}(x_k^i, u_k^i) \omega - d^{\epsilon}(x_k^i, u_k^i) ||_2^{2} + \lambda ||\omega||_F
\end{equation}

where the observed unknown dynamics $d^\epsilon$ is obtained by $d^{\epsilon}(x_k^i, u_k^i)=x_{k+1}^i-\hat{f}(x_k)-\hat{g}(x_k)u_k$. $\epsilon$ represents the noise, such as the motion noise. $t$ denotes the total number of training episodes, and $K$ is the number of time steps in the $i$-th training episode. The parameter $\lambda \in \mathbb{R}^{+}$ is a regularization hyperparameter used to prevent overfitting, and $||\omega||_F$ denotes the Frobenius norm of the weight matrix $\omega$.
Solving Eq.\eqref{eq: ridge regression}, the solution for the $\overline{\omega}_t$ can be obtained analytically, which is as below,
\begin{equation}\label{eq: var w}
    \Sigma_t^{\omega} = \Sigma_0 + \sum_{i=1}^{t-1} \sum_{k=1}^K \phi(x_k^i, u_k^i) \phi^\top(x_k^i, u_k^i)
\end{equation}

\begin{equation}\label{eq: mean w}
   \overline{\omega}_t = (\Sigma_t^{\omega})^{-1} \sum_{i=1}^{t-1} \sum_{k=1}^K \phi(x_k^i, u_k^i) d^{\epsilon}
\end{equation}
where $\Sigma_t^{\omega} \in \mathbb{R}^{P \times P}$ denotes the covariance matrix, with the initial covariance matrix given by $\Sigma_0= \lambda I_{P \times P}$. Substituting Eq.\eqref{eq: mean w} into the Eq.\eqref{eq: d}, the unknown dynamics prediction is,
\begin{equation}\label{eq: d prediction}
    d(x_*, u_*) = \underbrace{\overline{\omega}_t^\top \Phi(x_*)}_{\Xi(x_k^i, x_*, u_k^i)} [u^\top_* \quad 1]^{\top}
\end{equation}

where $\overline{\omega}_t$ is adopted from Eq.\eqref{eq: mean w} and $(x_*, u_*)$ is the testing input for robot unknown dynamics. $\Phi(x_*):=[\varphi_1(x_*), \cdots, \varphi_{m+1}(x_*)]$ is adopted where $\varphi_i$ is the RFF defined in Eq.\eqref{eq: RFF}. $\Xi(x_k^i, x_*, u_k^i)$ is constant given the training dataset $(x_k^i, u_k^i)$ and the known robot state input $x_* \in \mathbb{R}^n$ where $i=1,\cdots, t-1$ and $k=1,\cdots, K$. The control input $u_*$ is explicitly decoupled from the unknown dynamics, which is convenient for constructing an optimization problem for safety under CBF and enables an efficient, optimal solution. The details are in section \ref{sec: algorithm framework and analysis}.

\subsection{Efficient Data Collection}
Model-based reinforcement learning is sample-efficient at learning dynamics for a good policy, but it requires balancing exploration and exploitation. Otherwise, the learned model may get stuck in a local minimum, leading to poor reward performance. We adopt $\mathrm{LC}^3$ \cite{kakade2020information} to balance exploration and exploitation when collecting training data (paired robot state and control input) $\tau_i=\{(x_0^{i}, u_0^{i}), \dots,$ $(x_{K-1}^i, u_{K-1}^i), x_{K}^i \}$ under the total time step $K$ at episode $i$. Given the training dataset 
$\{\tau_i\}^{t-1}_{i=1}$,
the weight $\overline{\omega}_t$ and covariance $\Sigma_t^{\omega}$ are obtained via Eq.\eqref{eq: mean w} and Eq.\eqref{eq: var w} respectively. 

The uncertainty of the weight $\overline{\omega}_t$ is parameterized by $\Sigma_t^{\omega}$, which can be considered as a hyper-ball, i.e. $B^t$. To balance exploration and exploitation, Thompson Sampling (TS) for the robot's unknown dynamics $d$ is adopted for each exploration episode, which is $\tilde{\omega}_t \sim \mathcal{N}(\overline{\omega}_t, (\Sigma_t^{\omega})^{-1})$. Thompson Sampling (TS) obtains model weights from the distribution defined by the covariance matrix $\Sigma_t^{\omega}$, such that the sampled weights lie within the confidence ball $B^t$ with probability at least $1 - 2\delta^s$, where $\delta^s \in [0,1]$ denotes the failure probability of sampling within $B^t$. Further details can be found in \cite{luo2022sample}.

Unlike \cite{luo2022sample, kakade2020information} with RFF, this paper considers control-affine RFF for modeling robot dynamics, which has no intrinsic model error (control-affine form) for modeling the unknown dynamics. Although \cite{zhou2024safety} adopts the Quadrature Fourier Features (QFF) to model the state-dependent dynamics, it cannot explicitly decouple the control input into the control-affine form, whereas this paper can learn control-affine dynamics via control-affine RFF, which can facilitate solving the optimization problem under CBF. 

\subsection{Algorithm Framework and Analysis}\label{sec: algorithm framework and analysis}
Safe model-based reinforcement learning in Eq.\eqref{eq: safe learning opt} requires learning the unknown dynamics and generating the safe policy simultaneously, which is intractable in practice \cite{kakade2020information}. Thus, work in \cite{kakade2020information, luo2022sample, zhou2024safety} approximated the optimization problem (Eq.\eqref{eq: safe learning opt}) in two steps: learning $d(x_k, u_k)$ while obtaining optimal policy and obtaining safe policy, which is presented in Algorithm \ref{alg: algorithm}.

\textbf{Learning unknown dynamics while obtaining optimal policy.}
At episode $t$, we use the collected data to obtain $\overline{\omega}_t$ and then embed the sampled dynamics via Thompson Sampling into the Model Predictive Path Integral (MPPI) to balance exploration and exploitation. MPPI will generate the optimal policy $\pi^*$ using the sampled dynamics, which approximates Eq.\eqref{eq: learning objective}.

\textbf{Obtaining safe policy.} Let $u^{\mathrm{ref}}:= \bar{u}_k^t$ denote the nominal control input generated by MPPI. The corresponding safe control policy is then obtained by solving the following optimization problem,
\begin{align}\label{eq: CBF-ACP opt}
    &u^t_k = \argmin_{u_k^t} || u_k^t - u^{\mathrm{ref}}_k  ||^2 \notag \\ 
    \quad s.t.  \quad &h(x_{k+1}^t)-h(x_k^t)\geqslant-\gamma h(x_k^t)+S^{(r)} \quad \\
    &u_k^t\in[u_{\mathrm{min}}, u_{\mathrm{max}}] \notag
    \end{align}
where $x_{k+1}^t=\hat{f}(x_k^t)+\hat{g}(x_k^t)u_k^t+\overline{\omega}_t^\top \Phi(x_k^t)\begin{bmatrix} u_k^t \\1 \end{bmatrix}$. $u_{\mathrm{min}}$ and $u_{\mathrm{max}}$ are the maximum and minimum control limits allowed, respectively. The first constraint is from the uncertainty-aware safety barrier certificates from Eq.~\eqref{eq: prob B to B}.
    
Due to the employed control-affine RFF, the optimization problem (Eq.~\eqref{eq: CBF-ACP opt}) can be conveniently formulated using off-the-shelf tools, such as CVXPY\cite{diamond2016cvxpy}, because we don't need to propagate the optimization variable inside the data-driven model. Without considering the control-affine structure, Eq.~\eqref{eq: CBF-ACP opt} requires the propagation of the optimization variable inside the data-driven model, e.g. Gaussian Process, which leads Eq.~\eqref{eq: CBF-ACP opt} to be a nonlinear optimization that may not find the optimal solution efficiently.

\begin{algorithm}[t]
	\caption{SafeCARL: Safe reinforcement learning using control-affine dynamics}\label{alg: algorithm}
	\KwIn{Parameters: ACP failure probability $\alpha$ and learning rate $\delta$, $\gamma$ in CBF, number of features $P$, stage cost function $r$, initial state $x_0$ and the goal state $x_{\mathrm{goal}}$, training episode $T$, total time steps $K$ under each episodes}
	\KwOut{Control-affine $d(x_k, u_k)$ and $u^{\mathrm{safe}}$}
	\BlankLine
    \BlankLine
	\While{$t \leq T$}{
		Initialization $x_0^t \gets x_0 $ \;
		Sample $\tilde{\omega}_t \sim \mathcal{N}(\overline{\omega}_t, (\Sigma^{\omega}_t)^{-1})$ ; //  Exploration with Thompson Sampling\\
		\While{$k < K$}{
			$\pi_k^t \gets \argmin_{\pi} J^{\pi} (x_k^t;R, \tilde{\omega}_t)$ ; // $u^{\mathrm{reference}}$ by MPPI\\
            
			$u_k^t \gets \textnormal{Eq}.~\eqref{eq: CBF-ACP opt}$ \;
            $x_{k+1}^t \gets f(x_k^t) + g(x_k^t)u_k^t + \epsilon_k^t$ \;
			$d^{\epsilon}(x_k^t, u_k^t) \gets$ $x_{k+1}^t- (\hat{f}(x_k^t) + \hat{g}(x_k^t)u_k^t)$ \;
		}
		$\Sigma_{t+1}^{\omega} \gets \textnormal{Eq}.~\eqref{eq: var w}$ \;
        $\overline{\omega}_{t+1} \gets \textnormal{Eq}.~\eqref{eq: mean w}$ \;
	}
\end{algorithm}

\section{Results}
The simulations on the cartpole and 3D quadrotor are conducted to validate the sample efficiency, performance, and the safety guarantees under the proposed safe reinforcement learning framework.
To assess sample efficiency and performance, we compare against $\mathrm{LC}^{3}$ \cite{kakade2020information}, which employs RFF, whereas our method utilizes control-affine RFF (ARFF). To demonstrate the effectiveness of the proposed efficient uncertainty-aware safety constraint, we perform an ablation study under three settings: without the safety constraint (MPPI-ARFF), with the safety constraint (MPPI-ARFF-CBF), and with the uncertainty-aware safety constraint (our method). All algorithm parameters are kept identical unless otherwise noted. 

\subsection{Cartpole}
We use the cartpole dynamics from Gymnasium \cite{towers2024gymnasium}. The ground-truth cartpole dynamics can be written as,

\begin{equation}\label{eq: cartpole dynamics}
    \begin{aligned}
    \ddot{\theta} &=
\frac{G_r \sin\theta - M \cos \theta}
{l \left(\frac{4}{3} - \frac{m_p \cos^2\theta}{m_c+m_p} \right)} \\
\ddot{p} &= M - \frac{m_p l \ddot{\theta} \cos{\theta}}{m_c + m_p}
\end{aligned}
\end{equation}
where $M=\frac{u + m_p l \dot{\theta}^2 \sin{\theta}}{m_c+m_p}$. $G_r$ is the gravitational acceleration. $\theta\in\mathbb{R}$ is the angle of the pole, $p\in \mathbb{R}$ is the position of the cartpole, and $u\in \mathbb{R}$ is the control input. The robot state $x$ is defined as $x=[p, \dot{p}, \theta, \dot{\theta}]$.  $m_c=1.0$, $m_p=0.1$, and $l=0.5$ are the ground-truth parameter for simulation. In this simulation, we inject unknown dynamics by using inaccurate parameter values:  $m_c^{\mathrm{nom}}=1.5$, $m_p^{\mathrm{nom}}=0.05$, and $l^{\mathrm{nom}}=0.4$.

\textbf{Simulation setting.} The control task is to keep the cartpole upright ($x_{\mathrm{goal}}=[0, 0, 180^\circ, 0]$) via learning the unknown dynamics. The MPPI with the CBF is applied to the robot to ensure safety during exploration, as described in the Algorithm \ref{alg: algorithm}. The stage reward function $r$ for controlling the cartpole is $R = ||x-x_{\mathrm{goal}}||^2_Q + ||u||_{ R_c}^2$ where $Q=\mathrm{diag}(5, 0.1, 10, 0.1)$ and $R_c=\mathrm{diag}(0.01)$. The motion noise is $\epsilon \sim 0.001*\mathcal{N}(\textbf{0}_4, \textbf{I}_{4 \times 4})$.
The following MPPI hyperparameters are used throughout the experiments: a planning horizon of 50, 500 planning samples, a control variance of $5^2$, and a temperature parameter of $1.0$.

To safely learn the unknown dynamics of the cartpole system, we use $P=100$ features for both RFF and control-affine RFF (ARFF). The model training parameter is the same for fairness. The safety set is defined as $\mathcal{C}=\{[p, \dot{p}]| 1-\frac{p^2}{2.5^2} - \frac{\dot{p}^2}{3.0^2} \geqslant 0\}$. The CBF hyperparameter is set as $\gamma = 0.7$, and the ACP uses a failure probability of $\alpha = 0.02$. The simulation is repeated 100 times to obtain quantitative results for both safety and reward. The computation cost for solving Eq.~\eqref{eq: CBF-ACP opt} is presented in Table \ref{table: computation cost}.

\textbf{Reward and explore efficiency.}
To evaluate the sample efficiency and the performance of the proposed method, the MPPI with the ground-truth dynamics (MPPI-GT), the $\mathrm{LC}^3$ algorithms\cite{kakade2020information} under the random Fourier feature (RFF)  (MPPI-RFF), and our method without uncertainty-aware safety constraint (MPPI-ARFF) are applied for learning the unknown cartpole dynamics. Simulation results in Figure \ref{Fig: cartpole}(a) show that the reward achieved by our method is comparable to that of MPPI-GT and outperforms MPPI-RFF. On the other hand, our method converges in approximately 10 episodes, demonstrating its sample efficiency.

\textbf{Ablation study for safety.} 
We conduct experiments with MPPI-ARFF, MPPI-ARFF using only the CBF (MPPI-ARFF-CBF), and MPPI-ARFF with the uncertainty-aware CBF (i.e., MPPI-ARFF-CBF-ACP (our method), denoted as SafeCARL) to evaluate the effectiveness of the proposed uncertainty-aware safety module.

The results in Figure \ref{Fig: cartpole}(b) demonstrate that our framework can guarantee safety when learning the unknown dynamics, while MPPI-ARFF and MPPI-ARFF-CBF without considering the uncertainty can not guarantee safety. Additionally, we perform 100 experiments to empirically estimate the safety probability. As shown in Table \ref{table: safety}, our framework maintains safety within the predefined failure probability ($\alpha=0.02$).

\textbf{Exploration validation.} For learning the unknown dynamics, the robot must explore the environment both safely and efficiently. Figure \ref{Fig: cartpole}(c) illustrates the exploration trajectory of our proposed method, demonstrating that it effectively explores the environment while ensuring safety.

\begin{figure*}[t]
	\centering
	\subfigure[]
	{
		\centering
		\includegraphics[scale=1, width=0.31\linewidth]{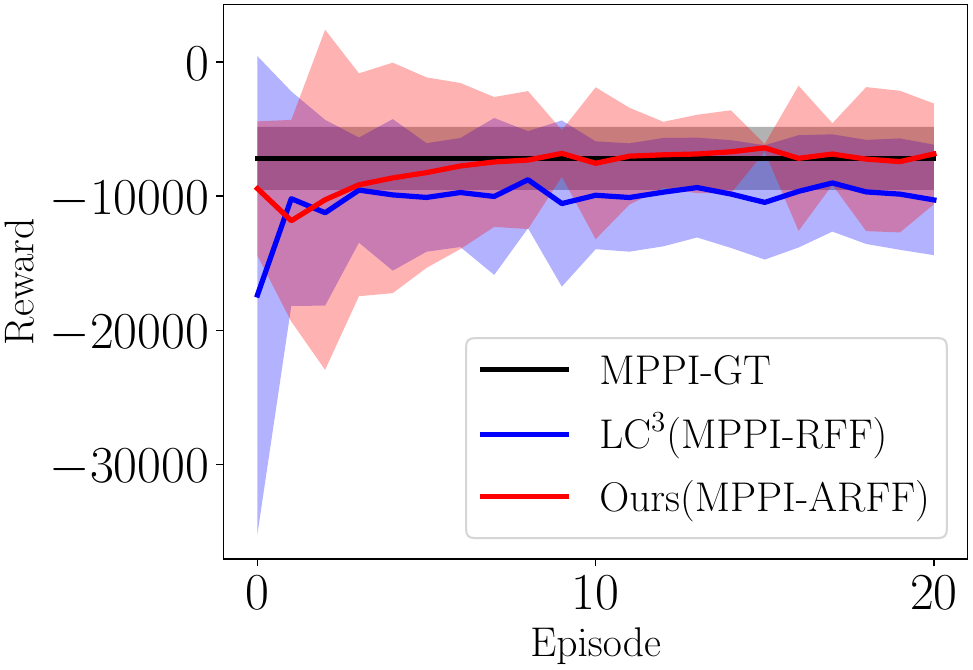}
	}
	\subfigure[]
	{
		\centering
		\includegraphics[scale=1, width=0.31\linewidth]{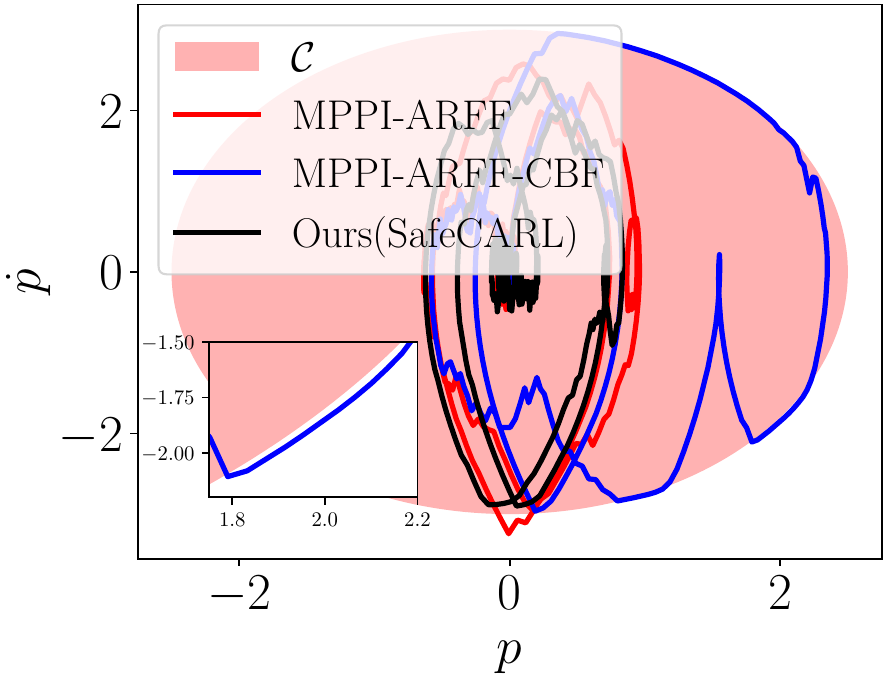}
	}
	\subfigure[]
	{
		\centering
		\includegraphics[scale=1, width=0.31\linewidth]{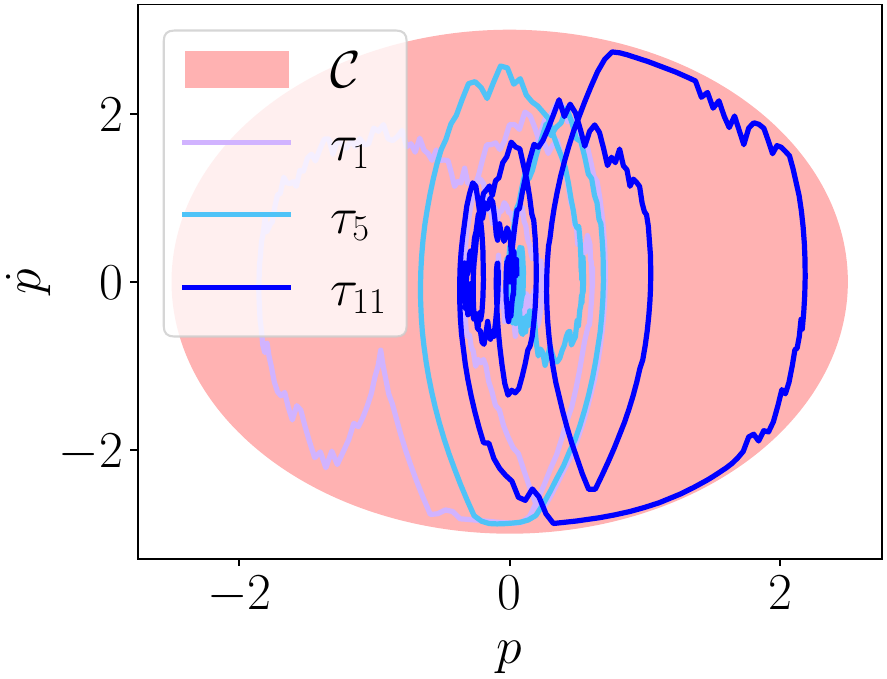}
	}
	\caption{\textbf{Cartpole}. (a)The reward under different algorithms. MPPI-GT is the reward from the ground truth dynamics while MPPI-ARFF is our method without the uncertainty-aware safety constraint. (b)Ablation study for the proposed framework. The start is $[0,0,180^\circ,0]$ and the goal is the $[0,0,0^\circ,0]$. $\mathcal{C}$ is the safety set. Safety cannot be guaranteed without an uncertainty-aware safety constraint under CBF.  (c)The safe exploration trajectory is generated from our method(SafeCARL), which can explore the environment efficiently while enforcing safety.}
	\label{Fig: cartpole}
    \vspace{-0.5 cm}
\end{figure*}

\subsection{3D Quadrotor Stabilization}\label{sec: quadrotor stabilization}

To validate the proposed method across different platforms, we use the 3D Quadrotor dynamics with differential flatness from \cite{wang2018safe} in this simulation.
\begin{equation}\label{eq: uav dynamics}
\left\{
\begin{aligned}
\ddot{p} &= G_r \mathbf{z}_\mathrm{w} + \mathrm{R}_{q}\, \mathbf{z}_{\mathrm{w}} f_z / m, \\[6pt]
\begin{bmatrix}
\dot{\psi} \\
\dot{\theta} \\
\dot{\Psi}
\end{bmatrix}
&=
\begin{bmatrix}
1 & s\psi \, t\theta & c\psi \, t\theta \\
0 & c\psi            & -s\psi \\
0 & s\psi \, sc\theta & c\psi \, sc\theta
\end{bmatrix}
\boldsymbol{w}_q.
\end{aligned}
\right.
\end{equation}

where $p\in \mathbb{R}^3$ is the position of the quadrotor in 3D space, $G_r$ is the gravitational acceleration, and $m$ is the quadrotor center mass. The ground truth for $m$ is $m=1.0$. $\mathbf{z}_\mathrm{w}=[0,0,1]^\top$ is adopted while $\mathrm{R}_q$ denotes the rotation matrix that transforms vectors from the body frame to the world frame, which can be found in \cite{wang2018safe}.  $\psi$(roll), $\theta$(pitch), and $\Psi$(yaw) are the  Euler angles to describe the orientation of the quadrotor w.r.t. the world frame. $s$, $c$, and $t$ represent the $\sin$, $\cos$, and $\tan$, respectively. While $sc$ represents the $\sec$. The robot state $x\in\mathbb{R}^9$ is $x=[p, \dot{p}, \psi, \theta, \Psi]$. Control input $\boldsymbol{w}_q\in \mathbb{R}^3$ is the angular velocity from the motor described under the world frame and control input $f_z\in\mathbb{R}$ is the thrust, i.e. $u=[\boldsymbol{w}_q, f_z]$.

\textbf{Simulation setting.}
In this simulation, we assume the unknown dynamics arise from inaccurate parameter values and environmental wind. Accordingly, the nominal mass is set to $m_\mathrm{nom}=0.8$, and the wind is modeled as $\sin{p}$. The nominal dynamics for the translation part in Eq.\eqref{eq: uav dynamics} becomes $\ddot{p} = G_r \mathbf{z}_\mathrm{w} + R_q \mathbf{z}_w f_z / m_{\mathrm{nom}}$. Followed by \cite{wang2018safe}, the safe reinforcement learning task for the quadrotor is to let the quadrotor reach the goal position, i.e. $p_\mathrm{goal}=[1.5, 2.0, 3.0]$, while satisfying the safety constraint $\mathcal{C}=\{[p_z, \dot{p}_z]| 1-\frac{p_{z}^2}{3.2^2} - \frac{\dot{p}^2_{z}}{3.5^2} \geqslant 0\}$ where $p_{z}$ and $\dot{p}_{z}$ represent the position and velocity of z axis, respectively. The motion noise is modeled as $\epsilon \sim 0.001*\mathcal{N}(\textbf{0}_9, \textbf{I}_{9 \times 9})$.

MPPI with the uncertainty-aware safety constraint is applied to achieve safe exploration. The stage cost function $R = ||x-x_{\mathrm{goal}}||^2_Q + ||u||_{ R_c}^2$ where $x_\mathrm{goal}=[p_\mathrm{goal}, \textbf{0}_6]$, $Q=\mathrm{diag}(10, 10, 10, 1, 1, 1, 1, 1, 1)$ and $R_c=0.01*I_{4\times4}$. 

The MPPI hyperparameters are used throughout the experiments: a planning horizon of 30, 500 planning samples, a control variance of $\mathrm{diag}(0.5, 0.5, 0.5, 2.0)$, and a temperature parameter of $1.0$.

We set $P=100$ for both control-affine RFF and RFF, and the model training parameter is the same for fairness. $\gamma=0.7$ is applied to CBF and failure probability $\alpha=0.02$ are adopted. The computation cost for solving Eq.~\eqref{eq: CBF-ACP opt} is presented in Table \ref{table: computation cost}.

\begin{table}[t]
\centering
\caption{Safety Evaluation: Minimal $h$ across Training Episodes}\label{table: safety}
\small
\setlength{\tabcolsep}{3pt}
\renewcommand{\arraystretch}{1.05}
\begin{adjustbox}{max width=\columnwidth}
\begin{tabular}{c |c |c |c |c}
\hline
 Method&  & $\mathrm{LC}^3$\cite{kakade2020information} & MPPI-ARFF-CBF & Ours(SafeCARL) \\
\hline
\multirow{2}{*}{Cartpole} & safe\%  & 0\% & 19\% & 100\% \\
\cline{2-5}
                          & min h & -1.69 $\pm$ 1.18  & -0.013 $\pm$ 0.015 &  0.042 $\pm$ 0.014 \\
\hline
\multirow{2}{*}{UAV}      & safe\%  & 0\% & 0\% & 98\% \\
\cline{2-5}
                          & min h & -0.26 $\pm$ 0.03 & -0.029 $\pm$ 0.006  &  0.01 $\pm$ 0.026\\
\hline
\end{tabular}
\end{adjustbox}
\vspace{0.5em}

{\footnotesize
    \parbox{\linewidth}{We conduct 100 experiments to evaluate different algorithms. MPPI-ARFF-CBF represents our proposed framework without uncertainty quantification. For each experiment, we record the minimal value of $h(x)$ across all training episodes. Safety is assessed by $h(x) > 0$, and the corresponding safe rate over the 100 experiments is summarized in the table. The failure probability under adaptive conformal prediction is set to $\alpha = 0.02$.}
    }
    \vspace{-0.6 cm}
\end{table}

\textbf{Reward and explore efficiency.} Similar to the setting of the cartpole under reward and exploration efficiency, we compare our method without uncertainty-aware safety constraint (MPPI-ARFF) against MPPI-GT and MPPI-RFF. As shown in Figure \ref{Fig: uav}(a), our approach using control-affine RFF outperforms MPPI-RFF and demonstrates high sample efficiency, achieving fast convergence.

\textbf{Ablation study for safety.} Similar to the setting of the cartpole under ablation study, we compare the method under the MPPI-ARFF, MPPI-ARFF-CBF, and MPPI-ARFF-CBF-ACP (our method named as SafeCARL). Figure \ref{Fig: uav}(b) demonstrates that the proposed method can guarantee safety when learning the unknown dynamics for a quadrotor.

\textbf{Applying the learned dynamics.} To leverage the learned dynamics for controlling the 3D quadrotor, we test our method under different initial positions toward the same goal. The resulting quadrotor trajectories under our method are shown in Figure \ref{Fig: uav}(c), demonstrating that our method successfully reaches the goal while maintaining safety.

\begin{figure*}[ht]
	\centering
	\subfigure[]
	{
		\centering
		\includegraphics[scale=1, width=0.31\linewidth]{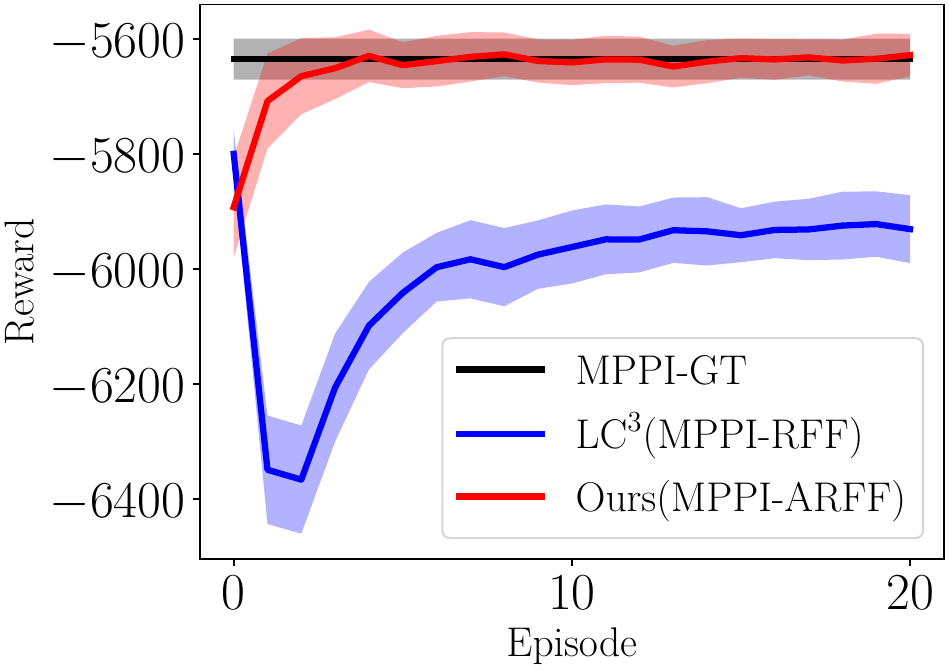}
	}
	\subfigure[]
	{
		\centering
		\includegraphics[scale=1, width=0.31\linewidth]{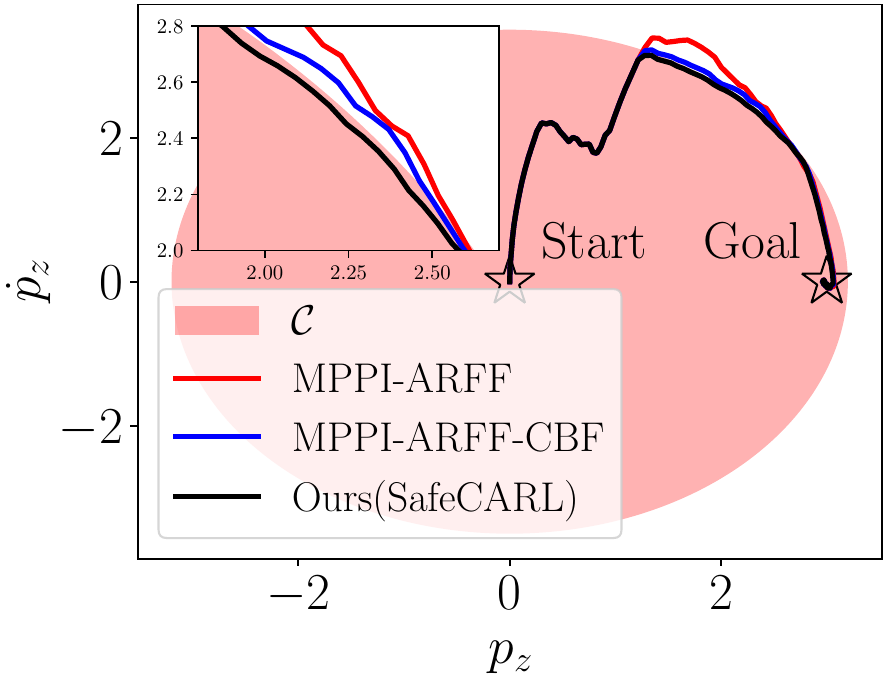}
	}
	\subfigure[]
	{
		\centering
		\includegraphics[scale=1, width=0.31\linewidth]{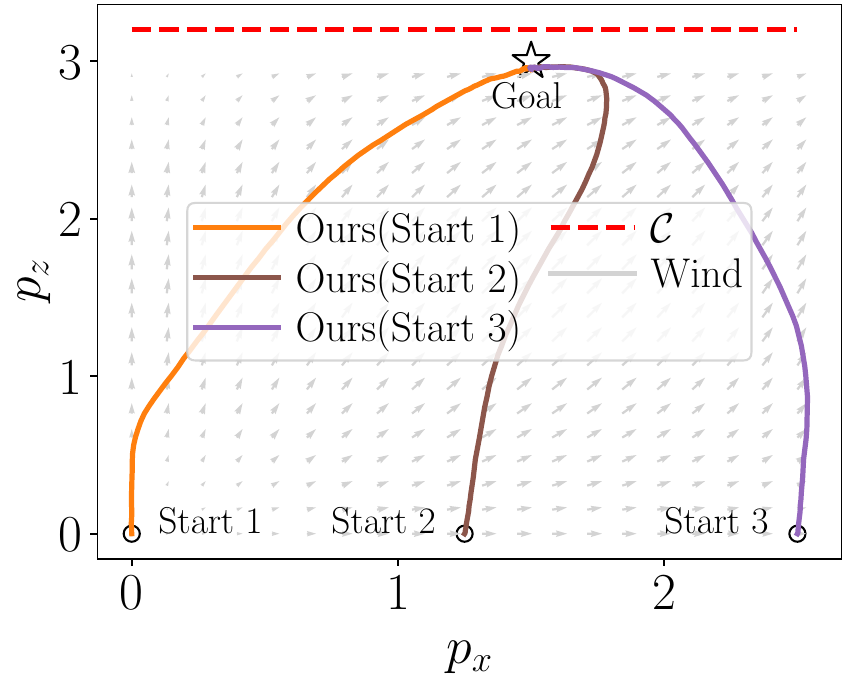}
	}
	\caption{\textbf{3D Quadrotor}. (a)The reward under different algorithms. MPPI-GT is the reward from the ground truth dynamics, while MPPI-ARFF is our method without the uncertainty-aware safety constraint. (b)Ablation study for the proposed framework. The start is the origin, and the target position is $[1.5, 2.0, 3.0]$. $\mathcal{C}$ is the safety set. Safety cannot be guaranteed without an uncertainty-aware safety constraint under CBF.  (c)We test the learned dynamics from different initial positions while continuing to apply the uncertainty-aware safety constraint to ensure safety. The resulting trajectories demonstrate that our method can generalize to various starting points while maintaining safety.}
	\label{Fig: uav}
    \vspace{-0.5 cm}
\end{figure*}

\begin{figure*}[t]
	\centering
	\subfigure[]
	{
		\centering
		\includegraphics[scale=1, width=0.31\linewidth]{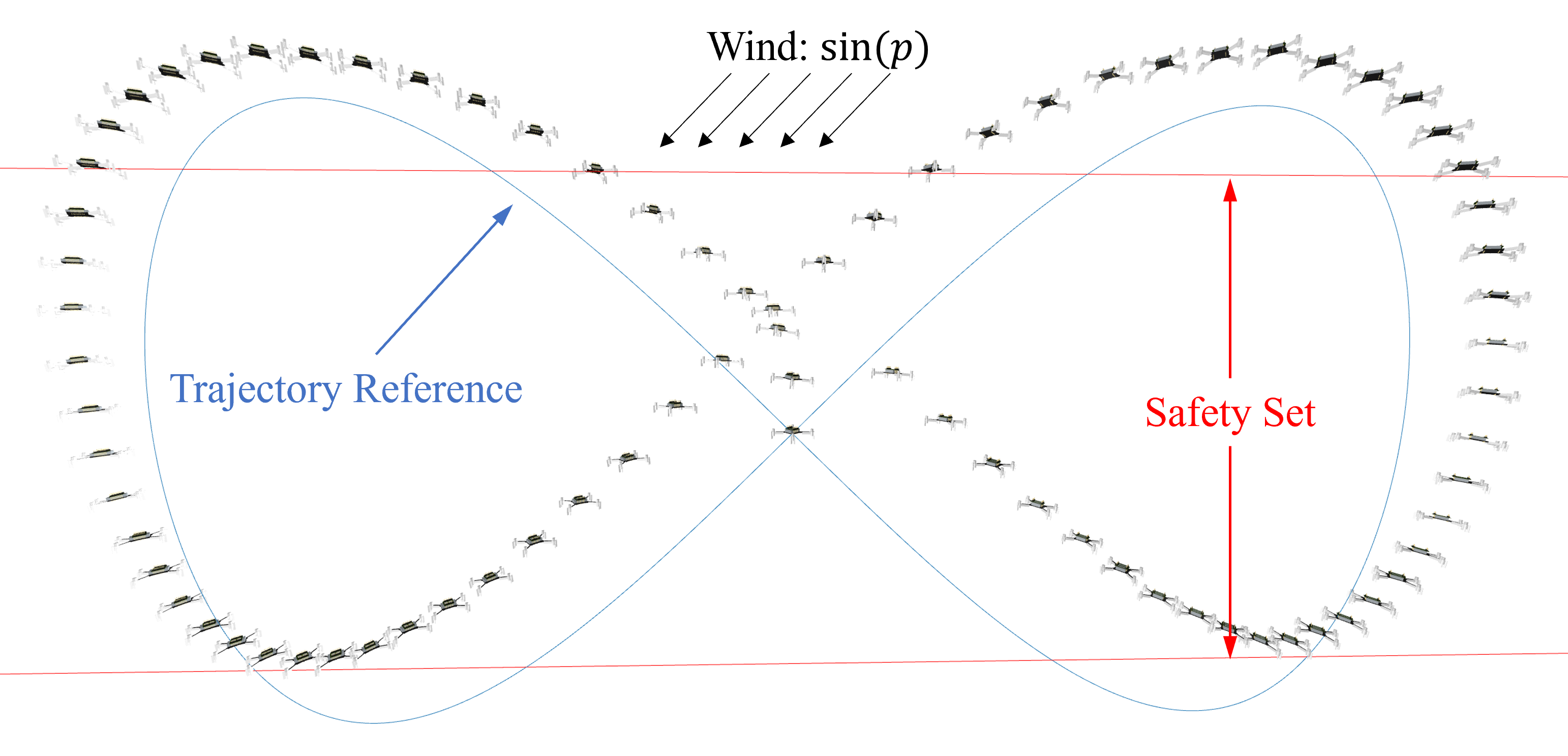}
	}
	\subfigure[]
	{
		\centering
		\includegraphics[scale=1, width=0.31\linewidth]{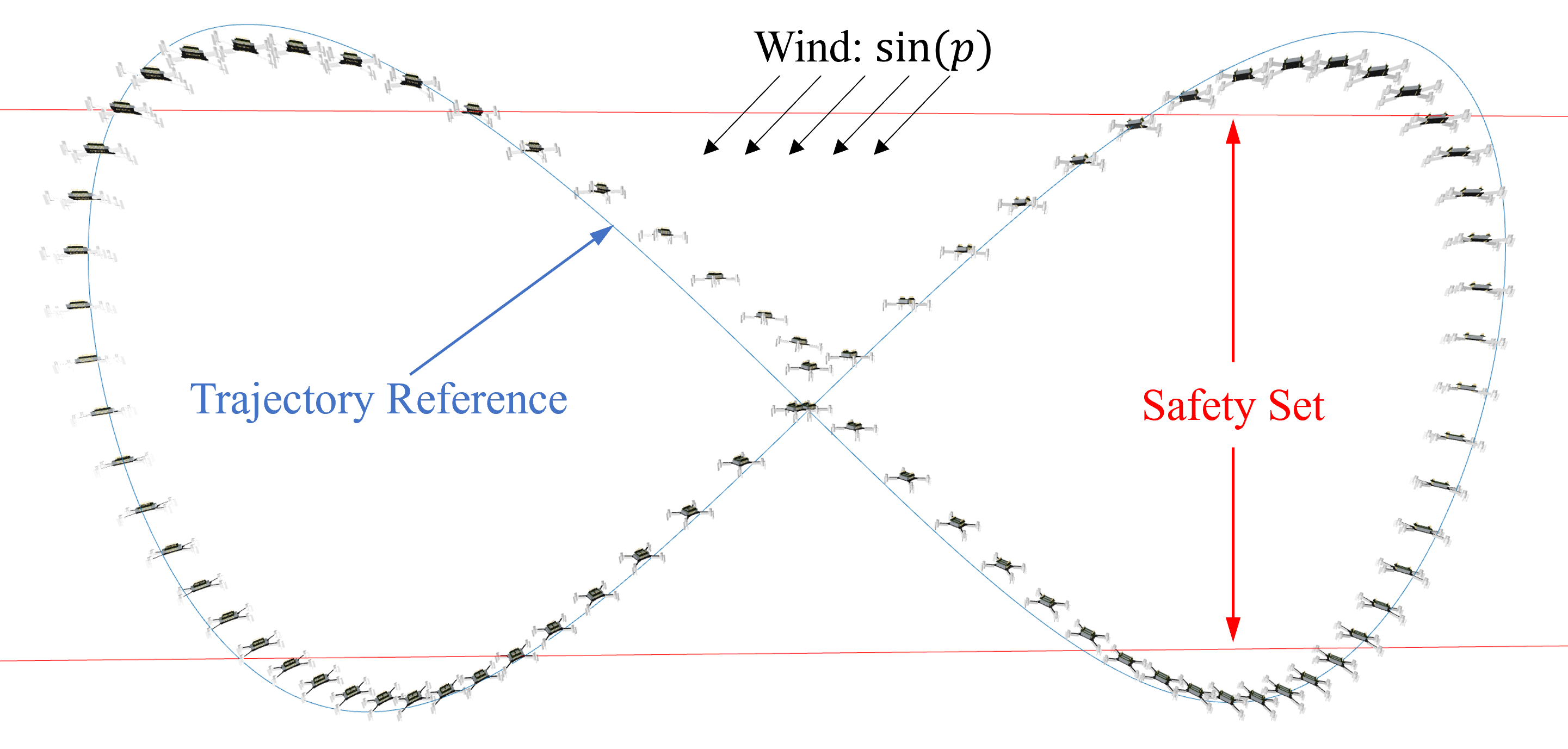}
	}
	\subfigure[]
	{
		\centering
		\includegraphics[scale=1, width=0.31\linewidth]{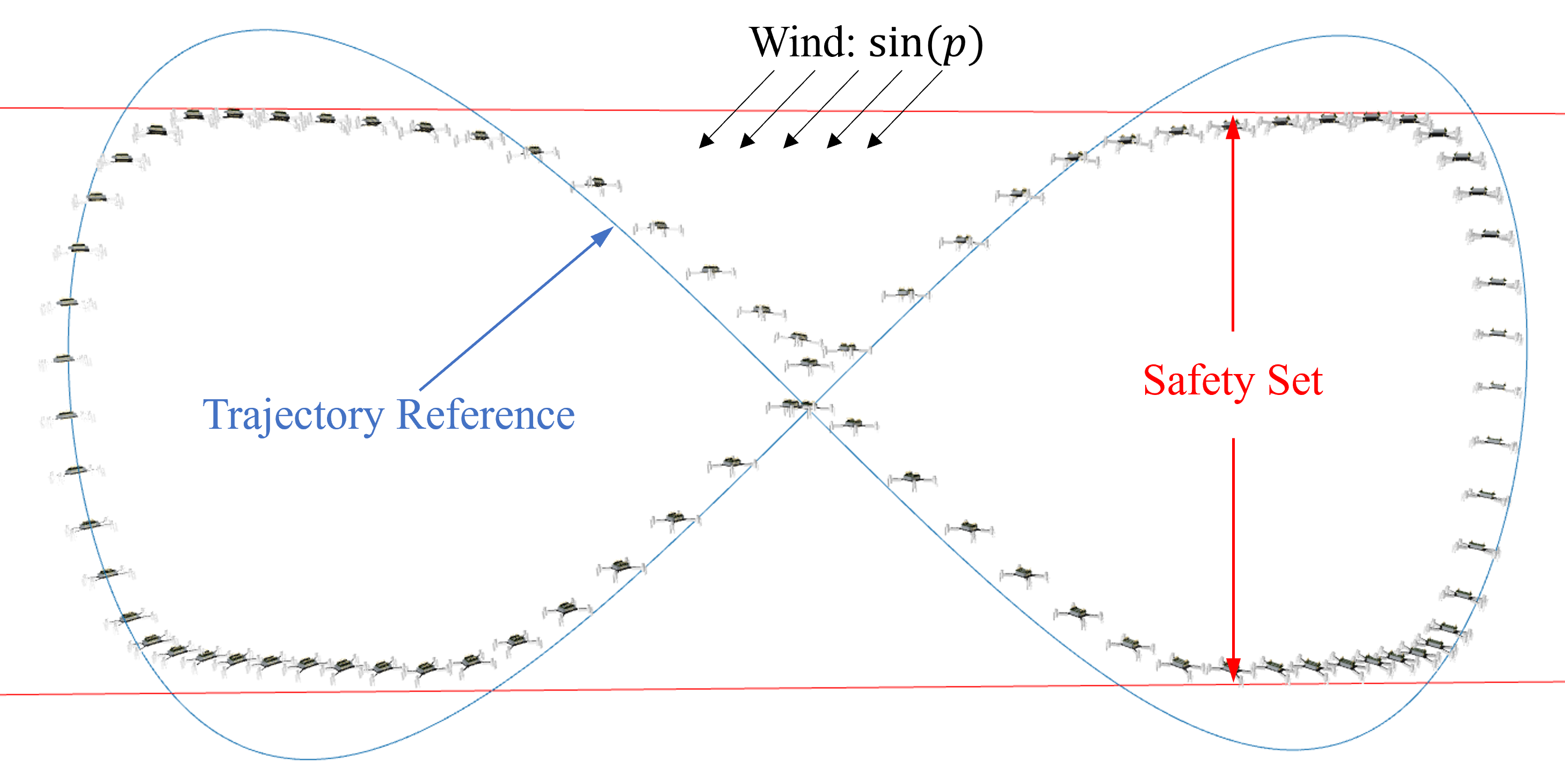}
	}
	\caption{\textbf{3D Quadrotor tracking.} The control task is to track the blue figure-eight trajectory while learning the unknown dynamics. All the trajectories of the result are from the final learning episode. (a) $\mathrm{LC}^3$\cite{kakade2020information} applies RFF to learn the unknown dynamics. (b) Our method under the ARFF without safety constraints. (c) Our method (SafeCARL). SafeCARL not only enforces safety but also achieves better track performance than $\mathrm{LC}^3$.}
	\label{Fig: uav_track}
    \vspace{-0.5 cm}
\end{figure*}

\begin{table}[h]
\centering
\caption{Computation Evaluation of ~Eq.\eqref{eq: CBF-ACP opt}}\label{table: computation cost}
\begin{tabular}{c |c |c }
\hline
 Method& MPPI-RFF-CBF-ACP & Ours(SafeCARL) \\
\hline
\multirow{1}{*}{Cartpole} & 2.51 $\pm$ 1.05 $\mathrm{ms}$ & 0.38 $\pm$ 0.16 $\mathrm{ms}$ \\
\hline
\multirow{1}{*}{UAV}      & 3.27 $\pm$ 1.66 $\mathrm{ms}$ & 0.44$\pm$ 0.08 $\mathrm{ms}$ \\
\hline
\end{tabular}
\vspace{0.5em}

{\footnotesize
    \parbox{\linewidth}{
    To evaluate the computational efficiency of SafeCARL against the method under RFF for dynamics learning, we measured the average execution time per learning episode required to solve the optimization problem in Eq.\eqref{eq: CBF-ACP opt}. Benchmarked on an AMD Ryzen 7960X 24-Core processor, SafeCARL leverages a convex formulation solved via CVXPY \cite{diamond2016cvxpy}, yielding approximately an eightfold speedup over the RFF-based method, which relies on the IPOPT solver \cite{wachter2006implementation}. Furthermore, in addition to its superior computational efficiency, SafeCARL empirically demonstrates enhanced control performance relative to the RFF-based baseline.
    }
    }
    \vspace{-0.5cm}
\end{table}

\subsection{3D Quadrotor Tracking}
We also adopt the same simulation settings as the 3D quadrotor stabilization to track the figure-eight trajectory shown in Figure \ref{Fig: uav_track}. The safety set is defined as $\mathcal{C}=\{[p_z, \dot{p}_z]| 1-\frac{(p_{z}-2)^4}{0.6^4} - \frac{\dot{p}^4_{z}}{1.8^4} \geqslant 0\}$. The control task is to track the figure-eight trajectory under unknown dynamics, which include uncertainties in the model parameters (i.e., mass) and environmental wind, assumed to be the same as in Section \ref{sec: quadrotor stabilization}.

We compare our method (SafeCARL) with the $\mathrm{LC}^3$ \cite{kakade2020information}. The simulation results in Figure \ref{Fig: uav_track} show that our method achieves higher tracking accuracy than $\mathrm{LC}^3$. This improvement arises from explicitly accounting for control-affine dynamics, which reduces potential model bias in the unknown dynamics and helps prevent overfitting. In addition to enhanced performance, our method ensures safety, whereas $\mathrm{LC}^3$ cannot.

\section{Conclusion}
This paper introduced a novel safe RL framework that incorporated a control-affine dynamics approximation method to model unknown dynamics amenable for safey control synthesis and used adaptive conformal prediction to derive an uncertainty-aware safety constraint. By integrating the constraints into an optimism-based exploration process, we allow for safe and efficient model learning while maximizing the control task efficiency.Simulation under complex dynamics validates the effectiveness of the proposed method. Future work will apply the proposed framework to physical robotic platforms in the real world.

\bibliographystyle{IEEEtran}
\bibliography{IEEEabrv,ref}

\end{document}